\documentclass[10pt]{article} 

\usepackage[accepted]{rlj}           

\usepackage{amssymb}            
\usepackage{mathtools}          
\usepackage{mathrsfs}           
\usepackage{graphicx}           
\usepackage{subcaption}         
\usepackage[space]{grffile}     
\usepackage{url}                
\usepackage{lipsum}             
\usepackage{algorithm}
\usepackage{algpseudocode}
\usepackage{caption}
\newcommand{\R}{\mathbb{R}}
\newcommand{\s}{\mathcal{S}}
\newcommand{\A}{\mathcal{A}}

\newcommand{\E}{\mathop{\mathbb{E}}}
\newcommand{\JA}[1]{{\color{red}{#1}}}

\usepackage{wrapfig}
\usepackage{amsmath}
\usepackage{amssymb}
\usepackage{mathtools}
\usepackage{amsthm}
\usepackage{url}
\usepackage{booktabs}
\theoremstyle{plain}
\newtheorem{theorem}{Theorem}
\newtheorem{proposition}{Proposition}
\newtheorem{lemma}{Lemma}
\newtheorem{innercustomgeneric}{\customgenericname}
\providecommand{\customgenericname}{}
\newcommand{\newcustomtheorem}[2]{%
  \newenvironment{#1}[1]
  {%
  \renewcommand\customgenericname{#2}%
  \renewcommand\theinnercustomgeneric{##1}%
  \innercustomgeneric
  }
  {\endinnercustomgeneric}
}
\definecolor{mycolor}{RGB}{19,106,212}
\newcustomtheorem{customlemma}{Lemma}
\newcustomtheorem{customtheorem}{Theorem}

\theoremstyle{definition}

\newtheorem{assumption}{Assumption}
\theoremstyle{remark}

\usepackage{tcolorbox}

\newcommand{\st}{\mathbf{s}_t}
\newcommand{\so}{\mathbf{s}}

\newcommand{\at}{\mathbf{a}_t}
\newcommand{\ao}{\mathbf{a}}

\title{Average-Reward Soft Actor-Critic}

\setrunningtitle{Average-Reward Soft Actor-Critic}

\author{Jacob Adamczyk\textsuperscript{1,2,$\dagger$},\ Volodymyr Makarenko\textsuperscript{3,$\dagger$}, \\Stas Tiomkin\textsuperscript{4},\ Rahul V. Kulkarni\textsuperscript{1,2}}

\emails{jacob.adamczyk001@umb.edu,volodymyr.makarenko@sjsu.edu,\\stas.tiomkin@ttu.edu, rahul.kulkarni@umb.edu}

\affiliations{
$^{1}$\textbf{Department of Physics, University of Massachusetts Boston}\\
$^{2}$\textbf{NSF Institute of Artificial Intelligence and Fundamental Interactions}\\
$^{3}$\textbf{Department of Computer Engineering, San Jos\'e State University}\\
$^{4}$\textbf{Department of Computer Science, Whitacre College of Engineering, Texas Tech University}
\par 
$^\dagger$ Equal contribution.
}

\contribution{
    We generalize the soft actor-critic (SAC) algorithm from the discounted to the average-reward setting. 
    }
    {
    \cite{SAC1} derived a MaxEnt RL algorithm, soft actor-critic, for the discounted setting. We derive theoretical results and implement new algorithmic techniques to adapt SAC to the average-reward setting.
    }

\contribution{
    We extend the policy improvement theorem to the entropy-regularized average-reward objective.
    }
    {
    Previous work demonstrated the policy improvement theorem separately in discounted MaxEnt RL~\citep{SAC1} and average-reward (un-regularized) RL~\citep{zhang2024implicit}. We close this gap by analyzing the theoretical properties of policy improvement in the entropy-regularized average-reward setting.
    }

\contribution{
    We experimentally demonstrate the advantage of our approach against available baselines in standard control environments.
    }
    {
    We compare our algorithm with existing baseline average-reward methods: {ARO-DDPG}~\citep{ARO-DDPG}, ATRPO~\citep{ARTRPO}, and APO~\citep{APO}.
    }

\def\abstracttext{
The average-reward formulation of reinforcement learning (RL) has drawn increased interest in recent years for its ability to solve temporally-extended problems without relying on discounting. 
Meanwhile, in the discounted setting, algorithms with entropy regularization have been developed, leading to improvements over deterministic methods. 
Despite the distinct benefits of these approaches, deep RL algorithms for the entropy-regularized average-reward objective have not been developed. While policy-gradient based approaches have recently been presented for the average-reward literature, the corresponding actor-critic framework remains less explored.
In this paper, we introduce an average-reward soft actor-critic algorithm to address these gaps in the field.
We validate our method by comparing with existing average-reward algorithms on standard RL benchmarks, achieving superior performance for the average-reward criterion.
}
\summary{
\abstracttext
}

\begin{document}

\makeCover  
\maketitle  

\begin{abstract}
\abstracttext
\end{abstract}

\section{Introduction}
A successful reinforcement learning (RL) agent learns from interacting with its surroundings to achieve desired behaviors, as encoded in a reward function. However, in ``continuing'' tasks, where the amount of interactions is potentially unlimited, the total sum of rewards received by the agent is unbounded. To avoid this divergence, a popular technique is to \textit{discount} future rewards relative to current rewards. The framework of discounted RL enjoys convergence properties \citep{sutton&barto, kakade2003sample, bertsekas2012dynamic}, practical benefits \citep{schulman2015high, andrychowicz2020matters}, and a plethora of useful algorithms \citep{nature-dqn, trpo,ppo, rainbow, SAC1} making the discounted objective an obvious choice for the RL practitioner. Despite these benefits, the use of discounting introduces a (typically unphysical) hyperparameter $\gamma$ which must be tuned for optimal performance. The difficulty in properly tuning the discount factor $\gamma$ is illustrated in our motivating example, Figure~\ref{fig:swimmer}. Furthermore, agents solving the discounted RL problem will fail to optimize for long-term behaviors that operate on timescales longer than those dictated by the discount factor, $(1-\gamma)^{-1}$. Moreover, recent work has argued that the discounted objective is not even a well-defined optimization problem \citep{naik2019discounted}. Importantly, despite most state-of-the-art algorithms operating within this discounted framework, their metric for performance is most often the total or average reward over trajectories, as opposed to the discounted sum, which they are designed to optimize. In such cases, the discounted objective is used as a crutch for optimizing the true object of interest: long-term average performance.

To address these issues, another objective for solving continuing tasks has been defined and studied \citep{schwartz1993reinforcement, mahadevan1996average}: the average-reward objective. Although it is arguably a more natural choice, it has less obvious convergence properties since the associated Bellman operators no longer possess the contraction property. Despite an ongoing line of work on the theoretical properties of the average-reward objective~\citep{zhang2021finite, wan2023learning}, there remain a limited number of deep RL algorithms for this setting. Current algorithms beyond the tabular or linear settings focus on policy-gradient methods to develop deep actor-based models~\citep{ARTRPO, APO, ARO-DDPG}. While these advancements represent a positive step toward solving the average-reward objective, there remains a need for alternative approaches for the problem of average-reward deep RL.
\begin{wrapfigure}{r}{0.42\textwidth}
  \begin{center}
    \includegraphics[width=0.42\textwidth]{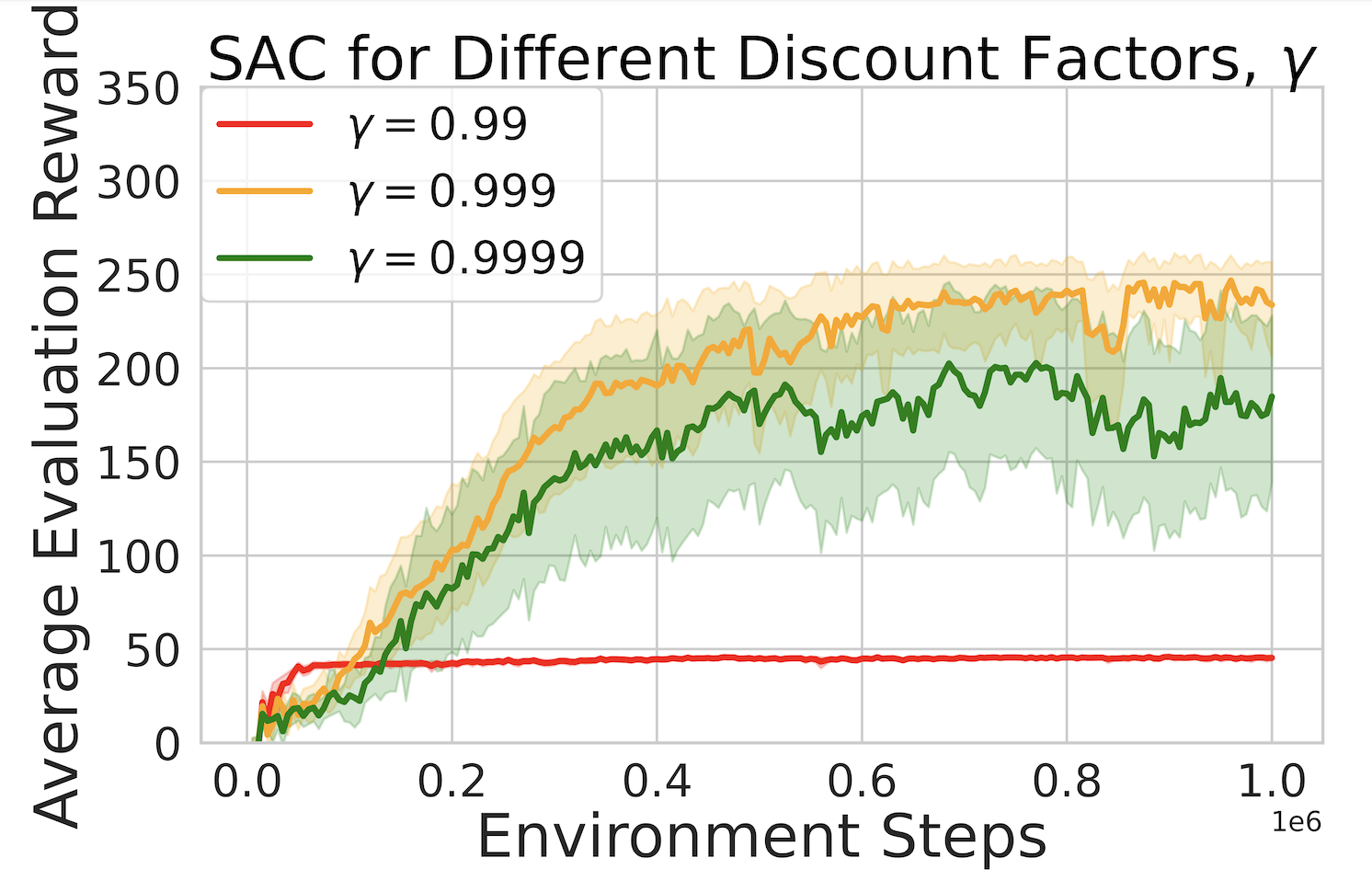}
  \end{center}
  \caption{The Swimmer-v5 environment, often not included in Mujoco benchmarks~\citep{swimmer-paper}, is notoriously difficult for discounted methods to solve when the discount factor is not tuned over and set to its default value of $\gamma=0.99$. Other discount-sensitive examples of environments have been discussed by~\cite{tessler2020reward}. We find that after carefully tuning the discount factor, SAC can solve the task, but the solution is quite sensitive to the choice of $\gamma$. Each curve corresponds to an average over 30 random seeds, with the standard error indicated by the shaded region.}
  \label{fig:swimmer}
\end{wrapfigure}
In both the discounted and average-reward scenarios, optimal policies are known to be deterministic \citep{mahadevan1996average, sutton&barto}. However, under various real-world circumstances (e.g. errors in the model, perception, and control loops), a deterministic policy can fail. In deployment, when RL agents face the sim-to-real gap, are transferred to other environments, or when perturbations arise \citep{HaarnojaSQL,haarnoja2018composable, eysenbach2022maximum}, fully-trained deterministic agents may be rendered useless. To address these important use-cases, it would be useful to have a stochastic optimal policy which is flexible and robust under uncertainty. Rather than using heuristics (e.g. $\varepsilon$-greedy, mixture of experts, Boltzmann) to generate a stochastic policy \textit{post-hoc}, the original RL problem can be regularized with an entropy-based term that yields an optimal policy which is naturally stochastic. Implementing this entropy-regularized RL objective corresponds to additionally rewarding the agent (in proportion to a temperature parameter, $\beta^{-1}$) for using a policy which has a lower relative entropy~\citep{levine2018reinforcement}, in the sense of Kullback-Leibler divergence. This formulation of entropy-regularized (often considered in the special case of maximum entropy or ``MaxEnt''\footnote{MaxEnt refers to using a uniform prior policy. In that case, ``low relative entropy'' (with respect to a uniform prior) is equivalent to ``high Shannon entropy''. In this work, we consider the case of more general priors.}) RL has led to significant developments in state-of-the-art off-policy algorithms \citep{HaarnojaSQL, SAC1, SAC2}. 

Despite the desirable features of both the average-reward and entropy-regularized objectives, an empirical study of the combination of these two formulations is limited, and no function-approximator algorithms exist yet for this setting. 
To address this, we propose a novel algorithm for average-reward RL with entropy regularization which is an extension of the discounted algorithm Soft Actor-Critic (SAC)~\citep{SAC1, SAC2}.
\nocite{blackwell1962discrete,mahadevan1996average}

Notably, our implementation requires minimal changes to common codebases, making it accessible for researchers and allowing for future extensions by the community.

\section{Preliminaries\label{sec:preliminaries}}
In this section, we discuss the background material necessary for the subsequent discussion.
Let $\Delta(\mathcal{X})$ denote the probability simplex over the space $\mathcal{X}$. A Markov Decision Process (MDP) is modeled by a state space $\s$, action space $\A$, reward function $r: \s \times \A \to \R$, transition dynamics $p: \s \times \A \to \Delta(\s)$ and initial state distribution $\mu \in \Delta(\s)$. The state space describes the set of possible configurations in which the agent (and environment) may exist. (This can be juxtaposed with the ``observation'' which encodes only the state information accessible to the agent. We will consider fully observable MDPs where state and observation are synonymous.) The action space is the set of controls available to the agent. Enacting control, the agent may alter its state. This change is dictated by the (generally stochastic) transition dynamics, $p$. 
At each discrete timestep, an action is taken and the agent receives a reward $r(s,a) \in \R$ from the environment. 

We will make some of the usual assumptions for average-reward MDPs~\citep{wan2021learning}:
\begin{assumption}
The Markov chain induced by any stationary policy $\pi$ is communicating. \label{assump:mc}
\end{assumption}

\begin{assumption}
The reward function is bounded.
\end{assumption}

In solving an average-reward MDP, one seeks a control policy $\pi$ which maximizes the expected \textit{reward-rate}, denoted $\rho^\pi$. In the average-reward framework, such an objective reads:
\begin{equation}
    \rho^\pi=\lim_{N \to \infty} \frac{1}{N}  \E_{\tau \sim{} p,\pi,\mu} \left[\sum_{t=0}^{N-1} r(\st, \at) \right],
    \label{eq:reward-rate-defn}
\end{equation}
where the expectation is taken over trajectories generated by the dynamics $p$, control policy $\pi$, and initial state distribution $\mu$.

The remaining non-scalar (that is, state-action-dependent) contribution to the value of a policy is called the average-reward differential bias function. Because of its analogy to the $Q$-function in discounted RL, we follow recent work \citep{ARTRPO} and similarly denote it as:
\begin{equation}
    Q^\pi_\rho(\so,\ao)=\E_{\tau \sim{} p,\pi} \left[ \sum_{t=0}^\infty  r(\st,\at) - \rho^\pi \Biggr| \so_0=\so,\ao_0=\ao \right].
    \label{eq:differential-value}
\end{equation}

We will now introduce a variation of this MDP framework which includes an entropy regularization term. For notational convenience we refer to entropy-regularized average-reward MDPs as ERAR MDPs. The ERAR MDP constitutes the same ingredients as an average-reward MDP stated above, in addition to a pre-specified prior policy\footnote{For convenience we assume that $\pi_0$ has support across $\A$, ensuring the Kullback-Leibler divergence is always finite.} ${\pi_0: \s \to \Delta(\A)}$ and ``inverse temperature'', $\beta$. The modified objective function for an ERAR MDP now includes a regularization term based on the relative entropy (Kullback-Leibler divergence), so that the agent now aims to optimize the expected \textit{entropy-regularized reward-rate}, denoted  $\theta^\pi$:
\begin{equation}
    \theta^\pi \doteq \lim_{N \to \infty} \frac{1}{N} \E_{\tau \sim{} p,\pi, 
    \mu} \left[ \sum_{t=0}^{N-1}  r(\st,\at) - \frac{1}{\beta} \log \frac{\pi(\at|\st)}{\pi_0(\at|\st)}  \right], \label{eq:theta-defn}
\end{equation}
\begin{equation}
    \pi^*(a|s) = \mathop{\mathrm{argmax}}_{\pi} \theta^\pi.
\end{equation}
Assumption~\ref{assump:mc} implies the expression in Equation~\eqref{eq:theta-defn} is independent of the initial state-action and ensures the reward-rate is indeed a unique scalar. From hereon, we will simply write $\theta=\theta^{\pi^*}$ for the optimal entropy-regularized reward-rate for brevity. Comparing to Equation~\eqref{eq:reward-rate-defn}, this rate is seen to include an additional entropic contribution, the relative entropy between the control ($\pi$) and prior ($\pi_0$) policies.

Beyond a mathematical generalization from the MaxEnt formulation, the KL divergence term has also found use in behavior-regularized RL tasks, especially in the offline setting \citep{wu2019behavior, zhang2024implicit} and has found growing interest in its application to large language models (LLMs)~\citep{rafailov2024direct, yan2024efficient}. Using a non-uniform prior has also been exploited to develop approaches for solving the un-regularized problem~\cite{adamczyk2025eval} and the problem of potential-based reward shaping and compositionality~\cite{adamczyk2023utilizing}. Intuitively, the choice of prior allows one to exploit inductive biases while maintaining robustness. 

The corresponding differential entropy-regularized action-value function is then given by:
\begin{equation}
    Q^\pi_\theta(\so,\ao) = r(\so,\ao) - \theta^\pi + \E_{\tau \sim{} p,\pi} \Biggl[ \sum_{t=1}^\infty \Biggl( r(\st,\at) 
     - \frac{1}{\beta} \log \frac{\pi(\at|\st)}{\pi_0(\at|\st)} - \theta^\pi \Biggr) \Biggr| \so_0=\so,\ao_0=\ao  \Biggr].
\label{eq:theta-differential-value}
\end{equation}

We have used the subscripts of $\theta$ and $\rho$ in this section to distinguish the two value functions. In the following, we drop the $\theta$ subscript as we focus solely on the entropy-regularized objective. Similar to the notation for the average-reward rate, we make the notation compact, and write ${Q(\so,\ao)=Q^{\pi^*}_\theta(\so,\ao)}$ as a shorthand.

\section{Prior Work}\label{sec:prior-work}

Research on average-reward MDPs has a longstanding history, dating back to seminal contributions by \cite{blackwell1962discrete} and later \cite{mahadevan1996average}, which laid the groundwork for future algorithmic and theoretical investigations \citep{even2009online, politex,abounadi,neu2017unified,wan2021learning}. Due to their theoretical nature, these studies primarily focused on algorithms within tabular settings or under linear function approximation, possibly explaining the limited work on the average-reward problem in the deep RL community. 
However, recent work has begun to address this challenge by tackling deep average-reward RL \citep{ARTRPO, APO, ARO-DDPG} with methods based on the policy gradient algorithm~\citep{sutton1999policy}.
Especially when tested on long-term optimization tasks, these studies have demonstrated superior performance of average-reward algorithms in the continuous control Mujoco benchmark~\citep{todorov2012mujoco}, compared to their discounted counterparts.

In the deep average-reward RL literature, research has primarily focused on extending known algorithms from the discounted to the average-reward setting. For example, \cite{ARTRPO}~first provided an extension of the on-policy trust region method TRPO~\citep{trpo} to the average-reward domain. To extend the classical discounted policy improvement theorem to this domain, they introduced a novel (double-sided) policy improvement bound based on K\'emeny's constant (related to the Markov chain's mixing time). Experimentally, they illustrated the success of ATRPO against TRPO, especially for long-horizon tasks in the Mujoco suite. Shortly thereafter,~\citep{APO} introduced an analogue of PPO~\citep{ppo} for average-reward tasks with an extension of generalized advantage estimation (GAE) and addressing the problem of ``value drift'', again proving successful in experimental comparisons with PPO. Most recently,~\cite{ARO-DDPG} continued this line of work by extending DDPG~\citep{ddpg} to the average-reward domain with extensive supporting theory, including finite-time convergence analysis. The authors also demonstrate the improved performance of their algorithm, ARO-DDPG, against the previously discussed methods, thereby demonstrating a new state-of-the-art algorithm for the average-reward objective.

In parallel, the discounted objective has included an entropy-regularization term, discussed in works such as \citep{todorov, todorov-pnas, ziebart-thesis, rawlik-thesis, HaarnojaSQL, geist2019theory} which to our knowledge has not yet been introduced in a deep average-reward algorithm. The included ``entropy bonus'' term in these methods has found considerable use in the development of both theory and algorithms in distinct branches of RL research~\citep{haarnoja2018composable, eysenbach2022maximum, park2023controllability}. This innovation yields optimal policies naturally exhibiting stochasticity in continuous action spaces, which has led SAC \citep{SAC2} and its variants to become state-of-the-art solution methods for addressing the discounted objective.

However, there is limited work on the combination of average-reward and entropy-regularized methods, especially for deep RL. Recent work by \cite{rawlik-thesis,neu2017unified, rose2021reinforcement,li2022stochastic,PRR,wu2024inverse} set the groundwork for combining the entropy-regularized and average-reward formulations by providing supporting theory and validating experiments. We will leverage their results to address the problem of deep average-reward RL with entropy regularization, while introducing some new theoretical results. 
In the next section, we present our average-reward extension of soft actor-critic.

\section{Proposed Algorithm \label{sec:asac}}
 
We begin with a brief discussion of soft actor-critic (SAC), for which we derive new theoretical results and provide an algorithm in the average-reward setting.
SAC~\citep{SAC1} relies on iteratively calculating a value (critic) of a policy (actor) and improving the actor through soft policy improvement (PI). In the discounted problem formulation, soft PI states that a new policy (denoted $\pi'$) can be derived from the value function of a previous policy ($\pi$) with $\pi' \propto \exp{\beta Q^{\pi}(\so,\ao)}$, which is guaranteed to outperform the previous policy in the sense of (soft) $Q$-values: ${Q^{\pi'}(\so,\ao) > Q^{\pi}(\so,\ao)}$ for all $\so,\ao$ (cf. Lemma 2 of \citep{SAC1} for details). We will first show that an analogous result for policy improvement holds in the ERAR setting. Note that in the case of large state-action spaces, experimentally verifying such inequalities becomes intractable~\citep{naik2024reinforcement} and can be alleviated by instead comparing reward rates: scalar quantities which can (in principle) be efficiently evaluated with rollouts. 

Since the value of a policy is now encoded in the entropy-regularized average reward rate $\theta^\pi$ and \textit{not} in the differential value, the analogue to policy improvement ($Q^{\pi'} > Q^{\pi}$) is to establish the bound $\theta^{\pi'} > \theta^{\pi}$ for some construction of $\pi'$ from $\pi$. Indeed, as we show, the same Boltzmann form over the differential value leads to soft PI in the ERAR objective. We later give some intuition on how this result can be understood as the limit $\gamma \to 1$ of SAC. After establishing PI and the related theory in this setting we will present our algorithm, denoted ``ASAC'' (for average-reward SAC, and following the naming convention of APO~\citep{APO} and ATRPO~\citep{ARTRPO}).

\subsection{Theory}
As in the discounted case, it can be shown that the $Q$ function for a fixed policy $\pi$ satisfies a recursive Bellman backup equation\footnote{Equation~\eqref{eq:V-defn} is an extension of $V^\pi_{\text{soft}}$ in \citep{HaarnojaSQL} to the case of non-uniform prior policy.}. This proposition was also derived in the concurrent work of~\cite{wu2024inverse} which analyzed the ERAR problem in the inverse RL framework:

\begin{proposition}
Let an ERAR MDP with reward function $r(\so,\ao)$, policy $\pi$ and prior policy $\pi_0$ be given. Then the differential value of $\pi$, denoted $Q^\pi(\st,\at)$, satisfies 
    \begin{equation}
        Q^\pi(\st,\at) = r(\st,\at) - \theta^{\pi} + \E_{\so_{t+1} \sim{p}} V^\pi(\so_{t+1}),
        \label{eq:q-pe}
    \end{equation}
with the entropy-regularized definition of state-value function
\begin{equation}
    V^\pi(\st) = \E_{\at \sim{} \pi} \left[ Q^{\pi}(\st,\at) - \frac{1}{\beta} \log \frac{\pi(\at|\st)}{\pi_0(\at|\st)} \right]. \label{eq:V-defn}
\end{equation}
\label{lem:erar-backup}
\end{proposition}
For completeness, we give a proof of this result (and all others) in the Appendix. As in the discounted case, the proof exploits the recursive structure of Eq.~\eqref{eq:theta-differential-value}.

As mentioned above, in the average reward formulation, the metric of interest is the reward-rate. Our policy improvement result thus focuses on increases in $\theta^{\pi}$, generalizing the recent work of \cite{ARTRPO} to the entropy-regularized setting. We find that the gap between any two entropy-regularized reward-rates can be expressed in the following manner:
\begin{tcolorbox}[colback=mycolor!5!white,colframe=mycolor!75!black]
\begin{lemma}[ERAR Rate Gap]
Consider two policies $\pi, \pi'$ absolutely continuous w.r.t. $\pi_0$. Then the gap between their corresponding entropy-regularized reward-rates is:
\begin{equation}
    \theta^{\pi'} - \theta^{\pi} = \E_{\substack{\st\sim d_{\pi'} \\ \at\sim \pi'}} \left( A^\pi(\st,\at) - \frac{1}{\beta} \log \frac{\pi'(\at|\st)}{\pi_0(\at|\st)}\right),
    \label{eq:theta-diff}
\end{equation}
where $A^\pi(\st,\at)=Q^\pi(\st,\at) - V^\pi(\st)$ is the advantage function of policy $\pi$ and $d_{\pi'}$ is the steady-state distribution induced by $\pi'$.
\label{lem:theta-relation}
\end{lemma}
\end{tcolorbox}
\vspace{1em}

As a consequence of this result, we find that with the proper choice of the updated policy $\pi'$, the right-hand side of Equation~\eqref{eq:theta-diff} is guaranteed to be positive, implying that soft PI holds.
Using the Boltzmann form of a policy \citep{SAC1} with the differential $Q$-values as the energy function and the appropriate prior distribution ($\pi_0$), gives the desired result:
\vspace{1em}
\begin{tcolorbox}[colback=mycolor!5!white,colframe=mycolor!75!black]
\begin{theorem}[ERAR Policy Improvement]
\label{thm:theta-relation}Let a policy $\pi$ absolutely continuous w.r.t. $\pi_0$ and its corresponding differential value $Q^\pi(\st,\at)$ be given. Then, the policy 
\begin{equation}
    \pi'(\at|\st) \doteq \frac{\pi_0(\at|\st) e^{\beta Q^{\pi}(\st,\at)}}{\int  e^{\beta Q^{\pi}(\st,\at)} d\pi_0(\at|\st)}
    \label{eq:pi}
\end{equation}
achieves a greater entropy-regularized reward-rate. That is, $\theta^{\pi'} \geq \theta^{\pi}$, 
with equality only at convergence, when $\pi'=\pi=\pi^*$.
\label{thm:pi}
\end{theorem}
\end{tcolorbox}
\vspace{1em}
Upon convergence, Equation~\eqref{eq:theta-diff} is identically zero, with the optimal policy satisfying ${\pi^* \propto \exp \beta A^*(\st,\at)}$ as expected from the analogous discounted result.
We note that the corresponding result in Lemma~2 of \cite{SAC1} for SAC (which uses a uniform prior policy), involves the {\em total} value function. On the other hand, under the average-reward objective, the improved policy is calculated with the {\em differential} value function.
Intuitively, this result can be understood as the $\gamma \to 1$ limit of PI for SAC. Numerically, this can be seen as setting $\gamma=1$ and continuously subtracting the ``extensive'' contribution to the total value function throughout. This bulk contribution scales with the number of timesteps in an episode and is the result of accruing a per-timestep reward $\theta^{\pi}$. Since the same term accrues in the state- and action-value functions, it cancels in the numerator and denominator of Equation~\eqref{eq:pi}. In the case of SAC, the bulk contribution (essentially $N\theta^\pi$, for $N\gg 1$) is included in the value function and so a discount factor $\gamma<1$ is required to ensure that the total value function is bounded in the limit of large $N$ (in the sense of Equation~\eqref{eq:theta-defn}). In contrast, for the case of ASAC, the bulk contribution is automatically excluded from the corresponding evaluation (by definition), and the differential value function remains bounded in the limit of large $N$, obviating the need to introduce a discount factor. This intuition can be formalized through a Laurent series expansion; cf.~\cite{mahadevan1996average}.

To complete the discussion of convergence for ASAC, the policy evaluation (PE) step must also converge. To formulate this, we rely on the work of~\cite{wan2021learning} who give convergence proofs for average-reward policy evaluation.
\begin{lemma}[ERAR Policy Evaluation]
Consider a fixed policy $\pi$, for which $\theta^\pi$ of Equation~\eqref{eq:reward-rate-defn} has been calculated (e.g. with direct rollouts). The iteration of Equations \eqref{eq:differential-value} and \eqref{eq:V-defn} converges to the entropy-regularized differential value of $\pi$: $Q^\pi(\st,\at)$.
\end{lemma}
\begin{proof}
The proof follows from the convergence results established in the un-regularized case, e.g. \cite{wan2021learning}. Since the policy $\pi$ is fixed (and $\pi \ll \pi_0$), the entropic cost $-\beta^{-1} \textrm{KL}\left({\pi||\pi_0}\right)$ is finite and can be absorbed into the reward function's definition: $r \xleftarrow[]{} r-\beta^{-1} \textrm{KL}\left({\pi||\pi_0}\right)$, and the standard proof techniques apply. 
\end{proof}
\subsection{Implementation}
\begin{figure*}[t]
\begin{minipage}[t]{1.0\textwidth}
  \includegraphics[width=\linewidth]{avg_reward.png}
  \caption{Training curves on continuous control benchmarks. We compare our algorithm, average-reward soft actor-critic (ASAC), with the following baselines: average-reward off-policy deep deterministic policy gradient (ARO-DDPG), average-reward trust-region policy optimization (ATRPO), and average-reward policy optimization (APO). ASAC learns the fastest with the best asymptotic performance. Each curve corresponds to an average over 20 random seeds, with standard errors indicated by the shaded region.}
  \label{fig:mujoco}
\end{minipage}
\end{figure*}
As in SAC \citep{SAC1}, we propose to interleave steps of policy evaluation (PE) and policy improvement (PI) using stochastic approximation to train the critic and actor networks, respectively. We use a deep neural net with parameters $\psi$, and denote $Q_\psi$ as the ``online'' critic network (with trainable parameters), and denote $Q_{\bar{\psi}}$ as the ``target'' critic, updated periodically through Polyak averaging of the parameters.
To implement a PI step, we use the KL divergence loss to update the parameters $\phi$ of an actor network $\pi_\phi$ based on the policy improvement theorem (Equation~\eqref{eq:pi}):
\begin{equation}
    \mathcal{L}_\phi = \sum_{\st\in \mathcal{B}} \textrm{KL}\left(\pi_\phi(\cdot|\st) \biggr|\biggr| \frac{\pi_0(\cdot|\st)e^{\beta Q_\psi(\st,\cdot)}}{Z(\st)}\right)\;.
\end{equation}
Similar to SAC, the independence of parameters on the partition function $Z$ allows us to simplify this loss expression to the more tractable form:
\begin{equation}
    \mathcal{L}_\phi = \sum_{\st\in \mathcal{B}} \E_{\at\sim{} \pi_\phi} \left( \log \frac{\pi_\phi(\at|\st)} {\pi_0(\at|\st)} - \beta^{-1} Q_\psi(\st,\at) \right)\;.
\end{equation}
In practice, we also use the re-parameterization trick to efficiently propagate gradients through the actor model.
After updating the actor via soft policy improvement, we update the critic (differential value) by performing a policy evaluation step with actions sampled from the current actor network. The mean squared error loss is calculated by comparing the expected $Q$-value to the right-hand side of Equation~\eqref{eq:q-pe}:
\begin{equation}
    \mathcal{L}_\psi = \sum_{(\st,\at,r,\so_{t+1}) \sim{} \mathcal{B}}\biggr|Q_\psi(\st,\at) - \hat{y}(r, \theta; \bar{\psi}, \phi) \biggr|^2,
\end{equation}
where $\hat{y}$ is the target value, defined as:
\begin{equation*}
    \hat{y}(r, \theta; \bar{\psi}, \phi)=r - \theta +  \E_{\ao_{t+1}\sim{} \pi_\phi(\cdot|\so_{t+1})} \left[Q_{\bar{\psi}}(\so_{t+1},\ao_{t+1}) - \frac{1}{\beta} \log \frac{\pi_\phi(\ao_{t+1}|\so_{t+1})}{\pi_0(\ao_{t+1}|\so_{t+1})} \right].
\end{equation*}

To update the ERAR rate $\theta^\pi$, we define its target as the batch-wise mean of its definition in Equation~\eqref{eq:theta-defn}. We treat $\theta$ as a trainable parameter (using an Adam optimizer) and train it to minimize the residual error compared to this target value.

We adopt the double $Q$-learning paradigm ~\citep{fujimoto2018addressing, SAC1, ARO-DDPG} used in previous literature for reducing estimation bias: two critics are maintained, and the minimum $Q$-value is used at each state-action pair. Although the corresponding theory~\citep{fujimoto2018addressing} for the average-reward case has not been studied in detail, we found this to improve experimental performance. 
Understanding the effect of estimation bias is an interesting line of study for future work.

Unique to the average-reward objective is the \textit{family} of solutions to the Bellman equation. Rather than a unique solution, the average-reward Bellman equation gives the differential value function an additional degree of freedom: If $Q(\so,\ao)$ satisfies Eq.~\eqref{eq:theta-differential-value} then $Q(\so,\ao)+c$ is also a solution for all $c \in \mathbb{R}$. Section 4.1 of~\citep{APO} provides an interesting discussion on the learning of value functions with an additive bias and a related downstream ``value drifting problem'', which they correct with value-based regularization. Section 6 of~\citep{wan2021learning} provides a discussion on learning centered value functions  via an additionally learned corrective ``value function'' $F$. 
To correct for this additional degree of freedom in an off-policy way, we introduce a baseline for centering the value function. Since an entire family of value functions can solve the Bellman equation, to pin the value, we choose the solution which passes through the origin, by always subtracting the value $Q(s=0, a=0)$. This choice is arbitrary, but works well in practice. 
Compared to the proposed regularization, it does not require any additional hyperparameters. Since it is not centering the value function in the traditional sense, it does not require on-policy data, but in principle the constant shift can be recovered upon convergence via rollouts of the optimal policy.

Finally, in average-reward tasks with terminating states, previous work~\citep{ARTRPO} has introduced a ``reset cost'', giving a penalty to the agent for resetting the environment and treating the reset state $s\sim{}\mu(\cdot)$ as the next state to emulate a continuing task. Prior work has chosen a fixed reset cost ($-100$) which was deemed suitable in the environments tested. However, it is not reasonable to expect such penalties to be effective for tasks with different reward scales or dynamics (cf. Humanoid results in Appendix D of~\citep{ARTRPO}). As such, we introduce a novel adaptive reset cost: To ensure the penalty for resetting is commensurate with the accrued rewards, we simply take the mean of all rewards in the current batch that do not correspond to termination. We use a rolling average (with the same learning rate as used for $\theta$) to slowly adapt the penalty to the agent's policy. We note that learning (and even defining) an ``optimal'' reset cost is an open question, which calls for further study. More details on the implementation, as well as the pseudocode for ASAC can be found in the Appendix, Section~\ref{app:implementation}.

\section{Experiments}
To evaluate our new algorithm, we test ASAC on a set of locomotion environments of increasing complexity including HalfCheetah, Ant, Swimmer, Hopper, Walker2d, and Humanoid (all version $5$) from the Gymnasium Mujoco suite~\citep{todorov2012mujoco,
towers2024gymnasium}. We compare the performance (average evaluation return across 10 episodes) against the existing average-reward algorithms discussed in Section~\ref{sec:prior-work}: APO, ATRPO, and ARO-DDPG. For these baselines, we use the hyperparameters provided in the corresponding papers. While the focus of this paper is on a comparison of algorithms for the average-reward criterion, we also provide a comparison to the discounted algorithm SAC in the Appendix.
To alleviate the cost of hyperparameter tuning, we simply use the default values inherited from SAC. Further details on the implementation and hyperparameter selection can be found in {Section~\ref{app:hparams}}.
ASAC performs well compared to both off-policy (ARO-DDPG) and on-policy algorithms (ATRPO, APO).
To maximize performance of the ARO-DDPG baseline, we found it beneficial to use a replay buffer of maximum length (equal to number of environment interactions). 
Compared to ASAC, the baselines fail to solve the task in a meaningful way on some environments (Walker, Ant, Humanoid), highlighting the importance of maximum-entropy approaches for high-dimensional locomotion tasks, especially in the average-reward setting.
The results of these experiments are shown in Figure~\ref{fig:mujoco}. Our experiments show that ASAC represents a novel and effective algorithm for the average-reward setting.

\section{Discussion\label{sec:discussion}}
The motivation for developing novel algorithms for average-reward RL arises from the problems generally associated with discounting. When the RL problem is posed in the discounted framework, a discount factor $\gamma\in[0,1)$ is a required input parameter. However, there is often no principled approach for choosing the value of $\gamma$ corresponding to the specific problem being addressed. Thus, the experimenter must treat $\gamma$ as a hyperparameter. This reduces the choice of $\gamma$ to a trade-off between large values to capture long-term rewards and small values to capture computational efficiency which typically scales polynomially with the horizon, $H=(1-\gamma)^{-1}$~\citep{kakade2003sample}.

It is important to note that the horizon $H$ introduces a natural timescale to the problem, but this timescale may not be well-aligned with another timescale corresponding to the optimal policy: the mixing time of the induced Markov chain. For the discounted solution to accurately approximate the average-reward optimal policy, the discounting timescale (horizon) must be larger than the mixing time. Unfortunately, the estimation of the mixing time for the optimal dynamics can be challenging to obtain in the general case, even when the transition dynamics are known, making a principled use of discounting computationally expensive. 

Therefore, an arbitrary ``sufficiently large'' choice of $\gamma$ is often made (sometimes dynamically~\citep{wei2021learning, koprulu2024dense}) without knowledge of the relevant problem-dependent timescale. This can be problematic from a computational standpoint as evidenced by recent work \citep{jiang2015dependence, ppo, andrychowicz2020matters}. These points are illustrated in Figure~\ref{fig:swimmer} which showed the performance of SAC for the Swimmer environment with different choices of $\gamma$. For the widely used choice $\gamma=0.99$ the evaluation rewards are low relative to the optimal case, whereas the average rewards algorithms perform well (Fig.~\ref{fig:mujoco}), highlighting the benefits of using the average-reward criterion. After submission of this paper, we became aware of related work:~RVI-SAC~\cite{hisaki2024rvi}, which uses relative value iteration (RVI) to estimate the reward-rate. A detailed comparison of RVI-SAC and ASAC is left to future work.

In this work, we have developed a framework for combining the benefits of the average-reward approach with entropy regularization. In particular, we have focused on extensions of the discounted algorithm SAC to the average-reward domain.
By leveraging the connection of the ERAR objective to the soft discounted framework, we have presented the first solution to ERAR MDPs in continuous state and action spaces by use of function approximation. Our experiments suggest that ASAC compares favorably in several respects to their discounted counterparts: stability, convergence speed, and asymptotic performance. Our algorithm leverages existing codebases allowing for a straightforward and easily extendable implementation for solving the ERAR objective.

\section{Future Work}
The current work suggests multiple extensions for future exploration.
Beginning with the average-reward extension of SAC~\citep{SAC1}, further developments have been made~\citep{SAC2} including automated temperature adjustment, which we foresee as a straightforward extension for future work. 
As a value-based technique, other ideas from the literature such as $\text{TD}(n)$, REDQ \citep{redq}, DrQ \citep{kostrikov2020image}, combating estimation bias \citep{hussing2024dissecting}, or dueling architectures \citep{wang2016dueling} may be included. From the perspective of sampling, the calculation of $\theta$ can likely benefit from more complex replay sampling, e.g. PER \citep{schaul2015prioritized}. An important contribution for future work is studying the sample complexity and convergence properties of the proposed algorithm. We believe that the average-reward objective with entropy regularization is a fruitful direction for further research and real-world application, with this work addressing a gap in the existing literature.
\section*{Acknowledgements}
JA would like to acknowledge the use of the supercomputing facilities managed by the Research Computing Department at UMass Boston; the Unity high-performance computing cluster; and funding support from the Alliance Innovation Lab – Silicon Valley. VM would like to acknowledge the Charles W. Davidson College of Engineering for providing computing resources used in experiment evaluation. RVK and JA would like to acknowledge funding support from the NSF through Award No.  PHY-2425180. ST was supported in part by NSF Award No. 2513350, PAZY grant (195-2020), and WCoE, {Texas~Tech~U.} This work is supported by the National Science Foundation under Cooperative Agreement PHY-2019786 (The NSF AI Institute for Artificial Intelligence and Fundamental Interactions, http://iaifi.org/).

\bibliography{main}
\bibliographystyle{rlj}

\beginSupplementaryMaterials

\section{Proofs\label{app:proofs}}

\begin{customlemma}{\ref{lem:erar-backup}}[ERAR Backup Equation]
Let an ERAR MDP be given with reward function $r(\so,\ao)$, fixed evaluation policy $\pi$ and prior policy $\pi_0$. Then the differential value of $\pi$, $Q^\pi(\st,\at)$, satisfies 
    \begin{equation}
        Q^\pi(\st,\at) = r(\st,\at) - \theta^{\pi} + \mathbb{E}_{\so_{t+1} \sim{} p} V^\pi(\so_{t+1}),
        \label{eq:q-peA}
    \end{equation}
with the entropy-regularized definition\footnote{Equation~\eqref{eq:V-defnA} is an extension of $V^\pi_{\text{soft}}$ in \cite{HaarnojaSQL} to the case of a non-uniform prior policy.} of state-value function
\begin{equation}
    V^\pi(\st) = \mathbb{E}_{\at \sim{} \pi} \left[ Q^{\pi}(\st,\at) - \frac{1}{\beta} \log \frac{\pi(\at|\st)}{\pi_0(\at|\st)} \right]. \label{eq:V-defnA}
\end{equation}
\end{customlemma}
\begin{proof}
    
We begin with the definitions for the current state-action and for the next state-action value functions, respectively:
\begin{align*}
    Q^\pi(\st,\at)&=r(\st,\at) - \theta^\pi + \E_{p,\pi} \left[ \sum_{k=1}^\infty \left(  r(\so_{t+k},\ao_{t+k}) - \frac{1}{\beta} \log \frac{\pi(\ao_{t+k}|\so_{t+k})}{\pi_0(\ao_{t+k}|\so_{t+k})} - \theta^\pi\right) \right],\\
    Q^\pi(\so_{t+1},\ao_{t+1})&=r(\so_{t+1},\ao_{t+1}) - \theta^\pi + \E_{p,\pi} \left[ \sum_{k=2}^\infty \left( (\so_{t+k},\ao_{t+k}) - \frac{1}{\beta} \log \frac{\pi(\ao_{t+k}|\so_{t+k})}{\pi_0(\ao_{t+k}|\so_{t+k})} - \theta^\pi\right) \right].
\end{align*}

Re-writing $Q^\pi(\st,\at)$ by writing out the first term in the infinite sum and highlighting the terms of $Q^\pi(\so_{t+1},\ao_{t+1})$ in \JA{blue},
\begin{align*}
    Q^\pi(\st,\at)=r(\st,\at) - \theta^\pi +&\E_{p,\pi} \Biggr[  r(\so_{t+1},\ao_{t+1}) - \frac{1}{\beta} \log \frac{\pi(\ao_{t+1}|\so_{t+1})}{\pi_0(\ao_{t+1}|\so_{t+1})} - \theta^\pi + \\
    &\JA{\sum_{k=2}^\infty \left(  r(\so_{t+k},\ao_{t+k}) - \frac{1}{\beta} \log \frac{\pi(\ao_{t+k}|\so_{t+k})}{\pi_0(\ao_{t+k}|\so_{t+k})} - \theta^\pi\right)} \Biggr], \\
    Q^\pi(\st,\at)=r(\st,\at) - \theta^\pi + &\E_{\so_{t+1} \sim{} p,\ao_{t+1}\sim{}\pi} \left[ Q^\pi_\theta(\so_{t+1},\ao_{t+1}) - \frac{1}{\beta} \log \frac{\pi(\ao_{t+1}|\so_{t+1})}{\pi_0(\ao_{t+1}|\so_{t+1})} \right].
\end{align*}
Identifying the entropy-regularized state value function (as in the discounted setting) ${V(\st) = \E_{\at \sim{} \pi} \left[ Q^\pi(\st,\at) - \frac{1}{\beta} \log \frac{\pi(\at|\st)}{\pi_0(\at|\st)} \right]}$ completes the proof.
\end{proof}

\begin{customlemma}{\ref{lem:theta-relation}}[ERAR Rate Gap]
Consider two policies $\pi, \pi'$ absolutely continuous w.r.t. $\pi_0$. Then the gap between their corresponding entropy-regularized reward-rates is:
\begin{equation}
    \theta^{\pi'} - \theta^{\pi} = \E_{\substack{\st\sim d_{\pi'} \\ \at\sim \pi'}} \left( A^\pi(\st,\at) - \frac{1}{\beta} \log \frac{\pi'(\at|\st)}{\pi_0(\at|\st)}\right),
    \label{eq:theta-diffA}
\end{equation}
where $A^\pi(\st,\at)=Q^\pi(\st,\at) - V^\pi(\st)$ is the advantage function of policy $\pi$ and $d_{\pi'}$ is the steady-state distribution induced by $\pi'$.
\end{customlemma}
\begin{proof}
Working from the right-hand side of the equation, 
\begin{align*}
    &\E_{\st\sim{} d_{\pi'}, \at\sim{} \pi' } \left( A^\pi(\st,\at) - \frac{1}{\beta} \log \frac{\pi(\at|\st)}{\pi_0(\at|\st)}\right) = \E_{\st\sim{} d_{\pi'}, \at\sim{} \pi' } \left( Q^\pi(\st,\at) - V^\pi(\st) - \frac{1}{\beta} \log \frac{\pi'(\at|\st)}{\pi_0(\at|\st)}\right)\\
    &= \E_{\st\sim{} d_{\pi'}, \at\sim{} \pi' } \left( {\color{blue}{r(\st,\at)}} - \theta^\pi + \E_{\so_{t+1}\sim{}p} V^\pi(\so_{t+1}) - V^\pi(\st) {\color{blue}{- \frac{1}{\beta} \log \frac{\pi'(\at|\st)}{\pi_0(\at|\st)} }}\right) \\
    &= \theta^{\pi'} - \theta^\pi + \E_{\st\sim{} d_{\pi'}, \at\sim{} \pi'} \left( \E_{\so_{t+1}\sim{} p(\cdot|\st,\at)} V^\pi(\so_{t+1}) - V^\pi(\st) \right)\\
    &= \theta^{\pi'} - \theta^\pi.
\end{align*}
where we have used the definition
\begin{equation}
    \theta^{\pi'} =\E_{\st\sim{} d_{\pi'}, \at\sim{} \pi' } \left(  {r(\st,\at)}- \frac{1}{\beta} \log \frac{\pi'(\at|\st)}{\pi_0(\at|\st)}\right),
\end{equation}
and
\begin{equation}
    \E_{\st\sim{} d_{\pi'}} \E_{\at\sim{}\pi'} \E_{\so_{t+1}\sim{} p} V^\pi(\so_{t+1}) = \E_{\st\sim{} d_{\pi'}} V^\pi(\st),
\end{equation}
which follows given that $d_{\pi'}$ is the stationary distribution. In other words, $d_{\pi'}$ is an eigenvector of the transition operator $p(\so_{t+1}|\st,\at) \cdot \pi'(\ao_{t+1}|\so_{t+1})$.
\end{proof}

\begin{customtheorem}{\ref{thm:pi}}[ERAR Policy Improvement]
\label{thm:theta-relationA}Let a policy $\pi$ absolutely continuous w.r.t. $\pi_0$ and its corresponding differential value $Q^\pi(\st,\at)$ be given. Then, the policy 
\begin{equation}
    \pi'(\at|\st) \doteq \frac{\pi_0(\at|\st) e^{\beta Q^{\pi}(\st,\at)}}{\int  e^{\beta Q^{\pi}(\st,\at)} d\pi_0(\at|\st)}
    \label{eq:piA}
\end{equation}
achieves a greater entropy-regularized reward-rate. That is, $\theta^{\pi'} \geq \theta^{\pi}$, 
with equality only at convergence, when $\pi'=\pi=\pi^*$.
\end{customtheorem}

\begin{proof}
Let $\pi'$ be defined as above. Then 
\begin{equation}
    \frac{1}{\beta} \log \frac{\pi'(\at|\st)}{\pi_0(\at|\st)} = Q^{\pi}(\st,\at) - \frac{1}{\beta} \log \E_{a\sim{}\pi_0} e^{\beta Q^\pi(\st,\at)}.
\end{equation}
Using Lemma~\ref{lem:theta-relation},
\begin{align*}
    \theta^{\pi'} - \theta^\pi &= \E_{s\sim{} d_{\pi'}, a \sim{} {\pi'}} \left(A^\pi(\st,\at) -\frac{1}{\beta}\log\frac{\pi'(\at|\st)}{\pi_0(\at|\st)}\right)\\
    &= \E_{s\sim{} d_{\pi'}, a \sim{} {\pi'}} \left(Q^\pi(\st,\at)-V^\pi(s) -\frac{1}{\beta}\log\frac{\pi'(\at|\st)}{\pi_0(\at|\st)}\right)\\
     &= \E_{s\sim{} d_{\pi'}, a \sim{} {\pi'}} \left(\frac{1}{\beta} \log \E_{a\sim{}\pi_0} e^{\beta Q^\pi(\st,\at)}-V^\pi(s) \right)\geq 0\;,\\
\end{align*}
where the last line follows from the variational formula~\cite{mitter2000duality,theodorou2012relative},
\begin{equation}
    \frac{1}{\beta} \log \E_{a\sim{}\pi_0} e^{\beta Q^\pi(\st,\at)} = \sup_\pi \E_{a\sim{} \pi} \left( Q^{\pi}(\st,\at) - \frac{1}{\beta} \log \frac{\pi(\at|\st)}{\pi_0(\at|\st)}\right).
\end{equation}
\end{proof}

\section{Implementation Details \label{app:implementation}}

For all SAC runs, we used \cite{stable-baselines3} implementation of SAC with hyperparameters (beyond the default values) shown below in Section~\ref{app:hparams}. The finetuned runs here took $\sim{}3000$ GPU hours for all environments, ran on a variety of RTX series and A100 GPUs. Each run requires roughly $\sim{}1-10$ GB of RAM.
\subsection{Hyperparameters\label{app:hparams}}
In addition to the methods discussed in the main text, we also use gradient clipping (on critic network only), with the maximum gradient norm of $10$ for all experiments.

For all ASAC experiments, we use the same hyperparameters as~\cite{SAC1}: batch size of $256$, replay buffer size of $1\,000\,000$, hidden dimension of $256$ for each of $2$ hidden layers (actor and critic networks), Polyak averaging with coefficient $0.005$, train frequency and gradient steps of $1$ (train for one gradient step at each environment step). We use the Adam optimizer for actor, critic, and reward-rate with learning rates $10^{-4}, 5\times10^{-4}, 5 \times 10^{-3}$. We clip the critic network gradients with a maximum norm of $10$. The scale for reset penalties is chosen as $p_0=10$ (see pseudocode below). In all environments (for SAC and ASAC) we use $\beta=5$, except for Swimmer and Humanoid, for which we use $\beta=20$. Note that this is in line with the ``reward scale'' used in \citep{SAC1}. We found that hyperparameter sweeps can give better performance for individual environments, but the values listed above gave a strong performance universally.

We found the replay buffer size to be a sensitive hyperparameter for ARO-DDPG, in particular for maintaining its asymptotic performance. We chose the largest replay buffer for ARO-DDPG (equivalent to total environment interactions), but further tuning is left to future work as it is an expensive environment-dependent operation. We also note that beyond the default hyperparameters for ASAC described above, we did not perform any tuning, showcasing ASAC's robustness to hyperparameter choice. Future work may entail an extensive hyperparameter sweep and sensitivity analysis to further understand the robustness and maximize performance across various environments.
\begin{figure}
    \centering
    \includegraphics[width=1.0\textwidth]{sac_avg_reward.png}
    \caption{Comparison to SAC shows that our average-reward extension outperforms the original discounted SAC on the environments tested. It is worth recalling that SAC and ASAC are inherently designed to optimize different objectives (a discounted return and average reward, respectively), despite the prevalent use of SAC as a surrogate for optimizing the average reward. Nevertheless, we give a comparison between the two algorithms here for completeness. We note that the reward values are different than in earlier environment versions (as used in e.g.~\cite{SAC1}), as the result of an updated reward function and bug fixes (including changes to contact forces, control costs), described in detail here: \url{https://farama.org/Gymnasium-MuJoCo-v5_Environments}. }
    \label{fig:my_label}
\end{figure}

\begin{algorithm}[H]
\caption{Average-Reward Soft Actor-Critic (ASAC)}
\begin{algorithmic}[1]
\State Initialize policy parameters $\phi$, $Q$-function parameters $\psi_1, \psi_2$.
\State Initialize target parameters $\bar{\psi}_1 \gets \psi_1$, $\bar{\psi}_2 \gets \psi_2$.
\State Initialize learning rates and optimizers (Adam).
\State Initialize mini-batch size $b$, Polyak step-size $\tau$, temperature $\alpha$ (fixed), replay buffer $\mathcal{D}$.
\While{not converged}
    \State Observe state $s_t$ and sample action $a_t \sim \pi_\phi(\cdot \mid s_t)$
    \State Execute $a_t$, observe reward $r_t$ and next state $s_{t+1}$
    \State Store $(s_t, a_t, r_t, s_{t+1})$ in replay buffer $\mathcal{D}$
    \For{each gradient step}
        \State Sample mini-batch of $b$ tuples $(s_i, a_i, r_i, s'_i) \sim \mathcal{D}$
        \State Sample $a'_i \sim \pi_\phi(\cdot \mid s'_i)$
        \State Shift the target $Q$-functions to pass through the origin:
        \[
        Q_{\bar{\psi}_j}(s, a) \gets Q_{\bar{\psi}_j}(s, a) - Q_{\bar{\psi}_j}(0, 0)
        \]
        \State Use pessimistic estimate by taking the pointwise minimum:
        \[
        Q_{\bar{\psi}}(s, a) \gets \min_{j=1,2} Q_{\bar{\psi}_j}(s, a) \quad \forall (s, a)
        \]
        \State If episode terminated at $s'$, apply adaptive penalty: 
        \[
        r \gets r - p \quad \text{where} \quad p\gets (1-\tau)p + \tau \bar{p} \text{  and  } \bar{p} = p_0 \cdot \frac{1}{b} \sum_{r \sim \mathcal{B}} r
        \]
        \State Compute target:
        \[
        \hat{y}(r, \theta; \bar{\psi}, \phi) = r - \theta + \mathbb{E}_{a' \sim \pi_\phi(\cdot \mid s')} \left[ Q_{\bar{\psi}}(s', a') -  \beta^{-1} \log \frac{\pi_\phi(a' \mid s')}{\pi_0(a' \mid s')} \right]
        \]
        \State Update $Q$-functions (for $\psi \in \{\psi_1, \psi_2\}$) by minimizing:
        \[
        \mathcal{L}_\psi = \frac{1}{b} \sum_{(s, a, r, s') \sim \mathcal{B}} \left(Q_\psi(s, a) - \hat{y}(r, \theta; \bar{\psi}, \phi) \right)^2
        \]
        \State Update policy by minimizing (using pessimistic estimate of online networks $Q_\psi$):
        \[
        \mathcal{L}_\phi = \frac{1}{b} \sum_{s \in \mathcal{B}} \mathbb{E}_{a \sim \pi_\phi} \left[ \log \frac{\pi_\phi(a \mid s)}{\pi_0(a \mid s)} - \beta^{-1} Q_\psi(s, a) \right]
        \]
        \State Update ERAR rate $\theta$ by minimizing $\mathcal{L}_\theta = (\theta - \bar{\theta})^2$ where:
        \[
        \bar{\theta} = \frac{1}{b} \sum_{(s, a, r) \sim \mathcal{B}} \left[ r -  \beta^{-1} \log \frac{\pi_\phi(a \mid s)}{\pi_0(a \mid s)} \right]
        \]
        \State Update target networks:
        \[
        \bar{\psi}_j \gets \tau \psi_j + (1 - \tau) \bar{\psi}_j \quad \text{for } j = 1, 2
        \]
    \EndFor
\EndWhile
\end{algorithmic}
\end{algorithm}

\end{document}